\documentclass{article}

\usepackage[preprint]{corl_2026}

\usepackage[utf8]{inputenc}
\usepackage[T1]{fontenc}
\usepackage{hyperref}
\usepackage{url}
\usepackage{booktabs}
\usepackage{amsmath,amssymb,mathtools}
\usepackage{graphicx}
\usepackage{subcaption}
\usepackage{algorithm}
\usepackage{algpseudocode}
\usepackage{microtype}
\usepackage{tikz}
\usepackage{amsthm}
\usepackage[section]{placeins}

\newtheorem{theorem}{Theorem}
\newtheorem{corollary}[theorem]{Corollary}

\graphicspath{{../artifacts/}{figures/}}

\csname title\endcsname{Liquid Networks with Mixture Density Heads for Efficient Imitation Learning}

\author{
Nikolaus Correll\\
University of Colorado Boulder\\
\csname texttt\endcsname{ncorrell@colorado.edu}
}

\begin{document}

\maketitle

\begin{abstract}
We compare liquid neural networks with mixture density heads against diffusion policies on Push-T, RoboMimic Can, and PointMaze under a shared-backbone comparison protocol that isolates policy-head effects under matched inputs, training budgets, and evaluation settings. Across tasks, liquid policies use roughly half the parameters ($\sim$4.3M vs.\ $\sim$8.6M), achieve $2.4\times$ lower offline prediction error, and run $1.8\times$ faster at inference. In sample-efficiency experiments spanning 1\% to 46.42\% of training data, liquid models remain consistently more robust, with especially large gains in low-data and medium-data regimes. Closed-loop results on Push-T and PointMaze are directionally consistent with offline rankings but noisier, indicating that strong offline density modeling helps deployment while not fully determining closed-loop success. Overall, liquid recurrent multimodal policies provide a compact and practical alternative to iterative denoising for imitation learning.
\end{abstract}

\section{Introduction}
Neural Ordinary Differential Equations (Neural ODEs) are a promising foundation for imitation learning, as their continuous-time formulation provides an implicit bias toward representing robot dynamics without the brittleness of discrete recurrent networks. However, standard Neural ODEs suffer from expensive backpropagation through numerical integrators. Closed-form Continuous-time (CfC) networks~\cite{hasani2022cfc} resolve this by using closed-form update rules, enabling efficient training while preserving continuous-time inductive bias. Yet imitation learning requires handling multimodal action distributions: when multiple action sequences solve the same observation, deterministic approaches collapse to their average, which is often not a valid solution. We therefore combine liquid neural networks (which include CfC as core component) with mixture density network (MDN) heads for explicit probabilistic action distributions.

We compare liquid networks with diffusion policies---the current dominant approach in imitation learning---across three robotics tasks using a fair shared-backbone protocol that isolates policy head quality. Our contributions are: (1) a fair comparison protocol ensuring identical perception and context for both heads, (2) demonstration that liquid policies with half the parameters achieve $2.4\times$ better accuracy and $1.8\times$ faster inference, (3) closed-loop validation showing liquid advantages persist in deployment, and (4) analysis of why continuous-time multimodal recurrent modeling outperforms iterative denoising.

\subsection{Related Work}
Liquid neural networks, specifically LTC~\cite{hasani2021ltc} and CfC~\cite{hasani2022cfc}, introduce continuous-time temporal inductive bias without requiring expensive numerical solvers during deployment. These networks have been successfully applied as robot controllers in various contexts, offering a natural fit for continuous control tasks~\cite{hasani2021ltc}.

Diffusion policies~\cite{chi2023diffusion,ho2020ddpm} have become dominant in imitation learning but suffer from inference latency due to multiple denoising steps. Recent work attempts to address this through consistency distillation~\cite{prasad2024consistency}, one-step distillation~\cite{wang2024onedp}, partial-denoising streaming~\cite{hoeg2024sdp}, and pruning combined with distillation for on-device deployment~\cite{wu2025lightdp}. Other efficiency-focused alternatives include Energy Policies~\cite{jia2025energy}, hybrid state-space/diffusion models such as Mamba Policy~\cite{cao2024mamba}, and training-free test-time composition~\cite{cao2025gpc}.

Flow matching~\cite{lipman2023flowmatching} is a promising alternative to DDPM denoising because it learns deterministic transport trajectories rather than iterative score refinement. In robotics policy generation, however, it still relies on sequential integration at inference and does not provide the continuous-time controller inductive bias that CfC-based liquid models inherit from Neural ODE structure. We therefore treat flow matching as complementary progress on generative efficiency, not a replacement for liquid recurrent dynamics in low-data control regimes.

By contrast, a liquid neural network policy operates as a single forward pass with a single control update per timestep, dramatically simplifying the deployment pipeline while maintaining explicit multimodal action distributions through mixture density heads.

\section{Liquid Nets}
Liquid neural networks, specifically LTC and CfC architectures, provide a foundation for continuous-time temporal modeling. LTC layers model hidden-state dynamics with explicit time constants and can be interpreted as discretized continuous-time systems~\cite{hasani2021ltc}. CfC keeps the same intuition but employs a closed-form update rule, avoiding expensive numerical integration in the training loop~\cite{hasani2022cfc}. For practical robotics applications, CfC is often preferred because it is faster to train and deploy while maintaining the continuous-time inductive bias.

For hidden state $h_{t-1}$ and input $u_t$, define $z_t=[h_{t-1};u_t]$. A CfC cell updates according to:
\begin{align}
f_t &= \sigma(W_f z_t+b_f), \\
\tau &= \exp(\theta_\tau), \\
g_t &= \frac{f_t}{\tau+f_t+\epsilon}, \\
\hat{h}_t &= \tanh(W_c z_t+b_c), \\
h_t &= g_t\odot\hat{h}_t + (1-g_t)\odot h_{t-1}.
\end{align}
In practice, stacked CfC layers (typically 5 layers) provide strong temporal memory with modest parameter growth, making them parameter-efficient compared to attention-based backbones.

At a high level, this continuous-time formulation motivates why liquid policies can be more sample-efficient than iterative denoising policies: they learn a reusable trajectory generator rather than repeatedly refining a sample with many sequential updates. A formal theorem and proof sketch are provided in Appendix~\ref{app:theory}.

\section{Experimental Setup and Datasets}
We evaluate liquid and diffusion policies across three robotics tasks chosen to test generality across action dimensionalities, state structures, and control objectives. All experiments use identical window sizes, normalization, and evaluation protocols.

\paragraph{Push-T (Contact-Rich Manipulation).}
Push-T is a planar contact manipulation benchmark~\cite{chi2023diffusion} where the agent must push a T-shaped block to a target pose using low-dimensional state ($\text{dim}=5$: $x, y, \theta$ of block plus target $x, y$). No visual observations in the base task. After windowing, the dataset contains 24,208 windows (16,945 train / 3,631 val / 3,632 test). This task is multimodal: multiple short-horizon action sequences can move the block toward the goal, making it ideal for testing multimodal policy representations.

\paragraph{RoboMimic Can (High-Dimensional Manipulation).}
RoboMimic~\cite{mandlekar2021robomimic} Can task requires a simulated robot arm to move a cylindrical object from start to target. State dimension is 57 (7D arm configuration, gripper state, object pose, relative geometry), action dimension is 7 (6D end-effector velocity + gripper). The dataset contains 72,552 windows (52,230 train / 10,161 val / 10,161 test). This is a critical test of scalability to high-dimensional control. RoboMimic Can is substantially higher-dimensional than Push-T, making it a critical test of scalability.

\paragraph{PointMaze (Navigation with Bimodal Choices).}
PointMaze~\cite{fu2021d4rl} is a continuous-navigation task where a point-mass agent reaches a goal in a maze. State dimension is 8 (2D position, 2D velocity, 2D goal, agent radius), action dimension is 2 (velocity commands). The dataset contains 72,551 windows (50,785 train / 10,883 val / 10,883 test). PointMaze provides a navigation counterpoint to manipulation, with lower dimensionality but long-horizon bimodal choices (left vs. right around maze obstacles).

\paragraph{Common preprocessing.}
For all tasks, we window trajectories as:
\begin{equation}
\mathbf{O}_{t}=(o_{t-H_o+1},\dots,o_t),\qquad
\mathbf{A}_{t}=(a_{t+1},\dots,a_{t+H_p}),
\end{equation}
with $H_o=2$ (history window) and $H_p=16$ (prediction horizon). Observations and actions are normalized to $[-1,1]$ using min-max scaling:
\begin{equation}
    ilde{x}=2\frac{x-x_{\min}}{x_{\max}-x_{\min}}-1.
\end{equation}

\section{System Architecture for Evaluation}

We present a unified system architecture combining liquid neural networks with mixture density heads for imitation learning. The architecture consists of three key components: a shared perception and backbone encoder, a liquid-based policy head, and a diffusion-based policy head. Both heads operate on identical latent representations, enabling fair comparison of their representational capacity. The liquid head uses a compact CfC recurrent encoder (0.5$\times$ scale) paired with an autoregressive GRU decoder that outputs multimodal action distributions, while the diffusion head employs a full-scale DDPM with iterative denoising. This shared-backbone design isolates policy head quality from perception differences.

\subsection{Autoregressive Multimodal Decoder}
The liquid encoder generates a summary hidden state that initializes a recurrent decoder for autoregressive action generation. At decoding step $k$, the decoder state evolves according to:
\begin{align}
e_k &= \phi(a_{k-1}), \\
s_k &= \mathrm{GRUCell}(e_k,s_{k-1}),
\end{align}
where $\phi$ is a learned embedding function that maps the previous action $a_{k-1}$ into a low-dimensional representation. The GRUCell is a standard gated recurrent unit that maintains the decoder hidden state $s_k$ across timesteps. We use GRUCell rather than another liquid network layer for two practical reasons: first, the decoder operates on low-dimensional action inputs after embedding, making the extra expressiveness of liquid dynamics unnecessary; second, GRU is a standard, well-optimized building block that trains stably and quickly, reducing overall training time while maintaining accuracy.

From the decoder state, we predict a $K$-component Gaussian mixture over the action at the current step:
\begin{equation}
p(a_k\mid s_k)=\sum_{j=1}^{K}\pi_{k,j}\,\mathcal{N}\!\left(a_k;\mu_{k,j},\operatorname{diag}(\sigma^2_{k,j})\right).
\end{equation}
This mixture density formulation is crucial: it avoids mode-averaging that undermines mean-squared error in multimodal control settings. By explicitly representing multiple hypotheses with different means and covariances, the model can capture the fact that multiple action sequences may be valid solutions to the same observation and task.

\subsection{Fair Shared-Backbone Architecture and Comparison Protocol}

To isolate policy head quality and ensure a fair comparison, we employ a consistent shared-backbone architecture across all three tasks. Both the liquid and diffusion heads receive identical inputs, training, and evaluation, with differences in performance arising purely from the head architectures.

\paragraph{Shared-Backbone Components.}

The architecture has three key parts: (1) a \emph{shared encoder} consisting of a frozen vision encoder (for Push-T) or identity projection (for RoboMimic Can and PointMaze), followed by a shared transformer backbone that processes observations and fuses temporal context. This encoder produces a latent representation passed to both policy heads. (2) The \emph{liquid head} uses a 5-layer CfC recurrent encoder at $0.5\times$ scale, producing a hidden state that initializes an autoregressive GRU decoder. The decoder outputs a 5-component Gaussian mixture, enabling explicit multimodal action representation. (3) The \emph{diffusion head} is a full-scale DDPM at $1.0\times$ parameters, evaluated with 50 denoising steps (see protocol details below). The liquid model is intentionally constrained to approximately half the parameters of diffusion ($0.5\times$ vs $1.0\times$), allowing us to study parameter efficiency without the confound of simply having more capacity.

	extbf{Note:} In all evaluation protocols, $K$ refers to the number of samples drawn for best-of-$K$ evaluation (e.g., best-of-10 MSE), not the number of denoising steps in the diffusion model. The diffusion head always uses 50 denoising steps for sampling.

\paragraph{Comparison Protocol.}
Our evaluation isolates policy head quality by ensuring both heads use: the same frozen perception (if applicable), the same shared transformer backbone context (computed once, then passed to both), the same 16-step action horizon, the same action normalization and windowing, and identical sample budgets $K \in \{1, 2, 5, 10\}$ during evaluation. This means differences in metrics arise purely from head architecture, not from different inputs, preprocessing, or evaluation protocols. The protocol answers: \emph{given identical latent context, how well does each model represent the distribution of future actions?}

\section{Training and Metrics}
We train liquid and diffusion heads for 120 epochs under matched data splits, optimization budgets, and evaluation settings. Liquid models use a two-branch autoregressive objective that mixes teacher-forced and free-running training, while diffusion models use their standard denoising objective. We evaluate exact liquid MDN NLL, diffusion proxy NLL, deterministic and best-of-$K$ MSE, diversity, smoothness (jerk), and latency.

For sample-efficiency analysis, we train from scratch at log-spaced data fractions and evaluate on fixed held-out test sets. Full objective definitions, exact fractions, and protocol details are provided in Appendix~\ref{app:training-metrics}.

\section{Results}

Our evaluation strategy across all three tasks follows a consistent protocol designed to ensure fair comparison and capture the complete picture of policy quality. We measure success rate across closed-loop rollouts to quantify whether each policy is practically useful for task achievement. Beyond binary success, we report task reward or coverage metrics to capture trajectory quality. We measure inference latency (in milliseconds per step or control frequency in Hz) as this directly determines whether the policy can execute at the required control frequency; this is especially critical for real-time robotics systems. Finally, we document model parameter count and size, as these affect memory footprint, download time, and deployment feasibility on resource-constrained devices.
\\
\\
	extbf{Diffusion models are evaluated with 50 denoising steps.} This setting was chosen as a practical trade-off between accuracy and inference speed, and is consistent across all reported experiments.\\
We select the best checkpoint for each model using free-running validation loss, where the decoder feeds back its own predictions rather than ground-truth actions. This choice is critical: teacher-forced metrics are optimistic because they provide clean inputs and do not reflect the distribution shift at test time. Free-running validation directly measures whether the model operates autonomously, making it appropriate for deployment-focused model selection.

\subsection{Open-Loop Offline Performance}\label{sec:offline}

As shown in Table~\ref{tab:offline-tables-all}, liquid achieves $2.4$--$2.5\times$ lower best-of-$K$ MSE and superior NLL across Push-T, RoboMimic Can, and PointMaze with approximately half the parameters (4.3M vs 8.6M baseline). This accuracy advantage is consistent across deterministic MSE, sample-mean MSE, and distributional metrics.

\begin{table*}[!t]
    \centering
    \small
    \caption{Offline comparison after 120 epochs under the shared-backbone protocol.
    Liquid + MDN uses roughly \textbf{half the parameters} ($\approx$4.3\,M vs.\ 8.6\,M) and runs
    \textbf{1.8--2$\times$ faster} (195--252\,ms vs.\ 380--448\,ms inference time).
    It achieves \textbf{lower NLL on all three tasks}, reflecting a better-calibrated density model.
    \textbf{MSE} is comparable on Push-T (within 2\%), but Liquid + MDN is
    \textbf{$\approx$18$\times$ lower on RoboMimic Can} and \textbf{10$\times$ lower on PointMaze}.
    Bold entries mark the best value per metric.
    Params in millions; ms = per-trajectory wall-clock inference time.}
    \label{tab:offline-tables-all}
    \begin{tabular}{lrrrr@{\quad}rrrr}
        \toprule
        & \multicolumn{4}{c}{\textbf{Liquid + MDN}} & \multicolumn{4}{c}{\textbf{Diffusion}} \\
        \cmidrule(lr){2-5}\cmidrule(lr){6-9}
        \textbf{Dataset}
            & Params & NLL$\downarrow$ & MSE$\downarrow$ & ms$\downarrow$
            & Params & NLL$\downarrow$ & MSE$\downarrow$ & ms$\downarrow$ \\
        \midrule
        Push-T
            & \textbf{4.34} & \textbf{-6.999} & 0.000158 & \textbf{195}
            & 8.60 & -3.768 & \textbf{0.000155} & 381 \\
        RoboMimic Can
            & \textbf{4.36} & \textbf{-20.830} & \textbf{0.007} & \textbf{205}
            & 8.84 & -15.732 & 0.124 & 380 \\
        PointMaze
            & \textbf{4.34} & \textbf{-8.615} & \textbf{0.045} & \textbf{252}
            & 8.60 & -3.578 & 0.450 & 448 \\
        \bottomrule
    \end{tabular}
\end{table*}

Table~\ref{tab:offline-tables-all} shows consistent trends across tasks: liquid policies are more parameter-efficient and substantially faster, while maintaining lower error than diffusion in higher-dimensional settings (RoboMimic Can and PointMaze). Appendix~\ref{app:offline-plots} provides additional analysis: best-of-$K$ performance (where $K$ is the number of samples drawn for evaluation, not denoising steps) confirms $2.4$--$2.5\times$ liquid advantage at $K=10$; per-step horizon analysis confirms liquid maintains lower error throughout the full 16-step horizon; diversity--accuracy trade-offs indicate liquid achieves lower error at comparable diversity; and qualitative trajectory samples show tighter liquid clusters and clearer multimodal behavior (left-vs-right on PointMaze).

\begin{figure*}[!htb]
    \centering
    \includegraphics[width=0.95\linewidth]{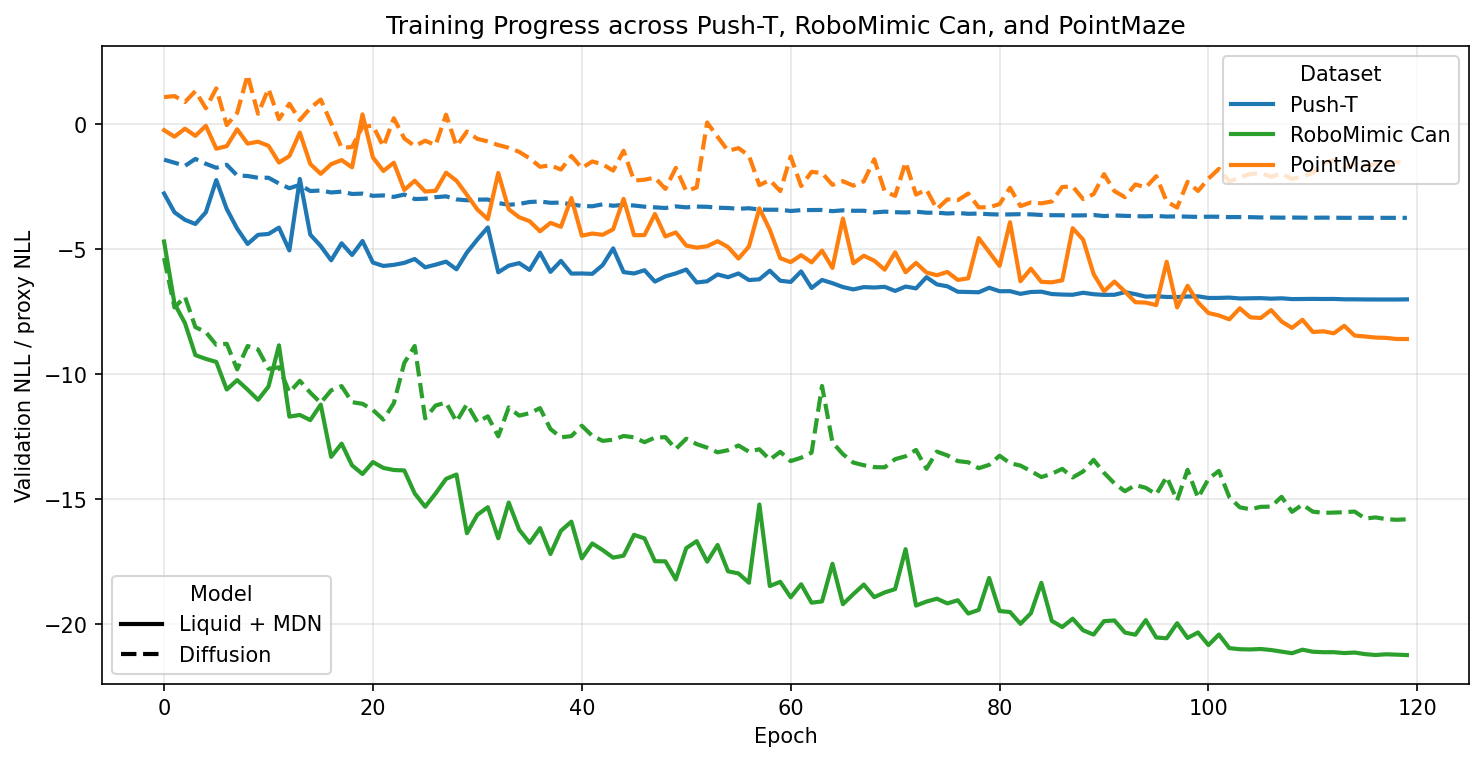}
    \caption{Validation negative log-likelihood over 120 epochs for all three datasets. The liquid model (0.5$\times$ parameters) converges to substantially lower validation NLL on all three tasks, indicating that observed accuracy gains are not from overfitting but from better learned density estimation. Both models train stably without early stopping.}
    \label{fig:training-all}
\end{figure*}

\subsection{Closed-Loop Validation}

We validate closed-loop deployment on Push-T and PointMaze to confirm that offline accuracy advantages translate to practical policy execution.

\paragraph{Closed-Loop Validation on Push-T (PyMunk simulator).}
We validate both models in closed-loop rollouts on Push-T using a physics-based PyMunk simulator with low-dimensional state observations. In a matched 100-episode rollout comparison, the liquid baseline achieves $91\%$ success with average max reward $0.9726$, while the diffusion baseline achieves $88\%$ success with average max reward $0.9811$. This confirms that predicted actions correspond to executable control behavior.

\paragraph{Closed-Loop Validation on PointMaze (Gymnasium Robotics).}
For PointMaze, we extend training to 240 total epochs under the same shared-backbone head architecture, windowing, and loss definitions, then evaluate in 50-trial $\times$ 20-episode rollouts using the official PointMaze Gymnasium \texttt{v3} environment (max horizon 300). Continuation training uses the offline D4RL/Minari PointMaze \texttt{v2} dataset, while evaluation is on Gymnasium \texttt{v3}, inducing a version-shift between training and deployment MDPs.

For clarity, PointMaze reports two complementary completion metrics. \emph{Success} uses the environment-provided success flag (goal reached under the benchmark's default radius), while \emph{Distance-Success} is a stricter geometric criterion: an episode counts as distance-successful if its minimum goal distance during rollout is at most $0.20$. This distinguishes coarse task completion from tighter goal-approach quality.

Liquid achieves higher success and distance-success rates than diffusion under the training--deployment version shift (see Table~\ref{tab:closedloop-results}) and achieves slightly lower final goal distance. We omit latency from the closed-loop table because wall-clock rollout time is dominated by simulator/control-loop pacing rather than model compute.

\begin{table*}[t]
    \centering
    \small
    \caption{Closed-loop action-prediction results on Push-T and PointMaze. Bold entries indicate the better model per metric within each task (higher is better for success, distance-success, and reward). Push-T uses a matched 100-episode comparison; PointMaze reports 50-trial aggregate means. For PointMaze, \emph{Success} uses the environment success flag, while \emph{Distance-Success} uses a stricter rollout minimum-distance threshold of $0.20$.}
    \label{tab:closedloop-results}
    \begin{tabular}{llrrr}
        \specialrule{0.08em}{0pt}{0pt}
        {\bfseries Task} & \textbf{Model} & \textbf{Success (\%)}$\uparrow$ & \textbf{Distance-Success (\%)}$\uparrow$ & \textbf{Reward}$\uparrow$ \\
        \midrule
        Push-T & Liquid + MDN & \textbf{91.0} & -- & 0.9726 \\
        Push-T & Diffusion & 88.0 & -- & \textbf{0.9811} \\
        \midrule
        PointMaze & Liquid + MDN & \textbf{20.0} & \textbf{9.7} & \textbf{7.71} \\
        PointMaze & Diffusion & 9.5 & 3.7 & 6.48 \\
        \bottomrule
    \end{tabular}
\end{table*}

\paragraph{Sample Efficiency Across Data Regimes.}

\begin{figure*}[!htb]
    \centering
    \includegraphics[width=\linewidth]{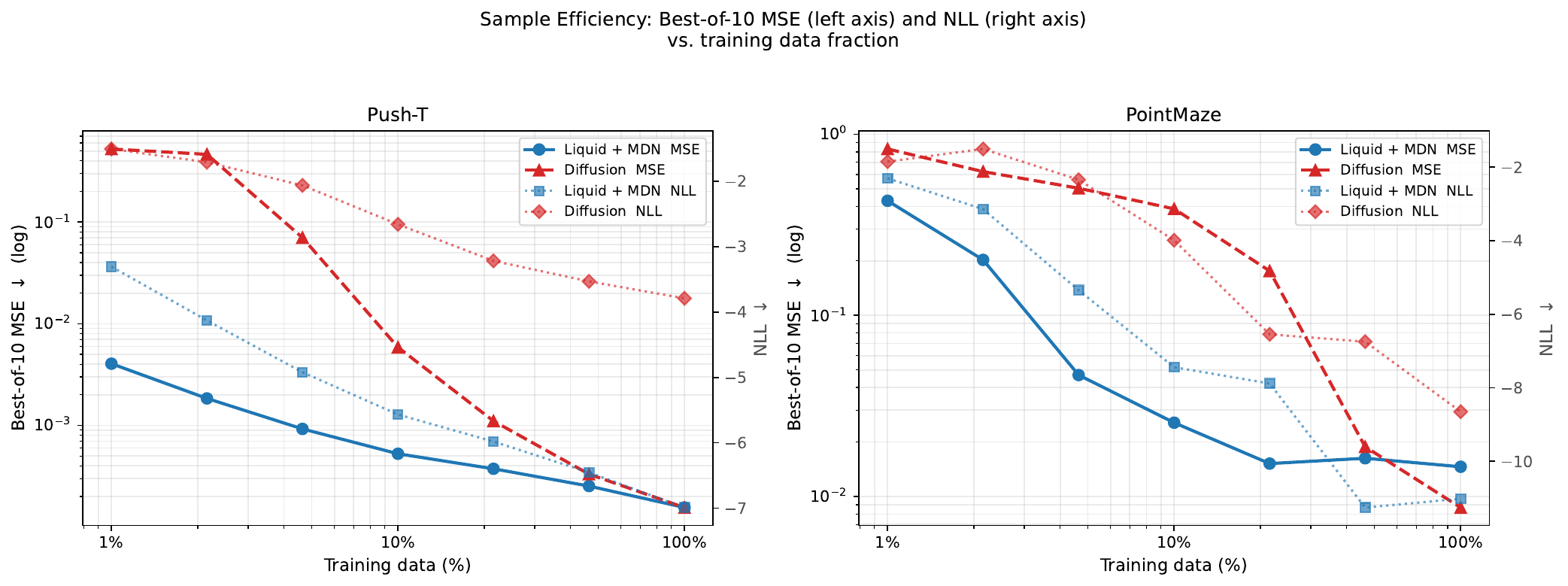}
    \caption{Sample efficiency for Push-T (left) and PointMaze (right). Solid lines show best-of-10 MSE (left axis, log scale); dotted lines show NLL (right axis). Blue = Liquid + MDN, red = Diffusion. Liquid consistently achieves lower MSE and better NLL at every training fraction, with the largest advantage in the data-scarce regime. NLL trends mirror MSE: liquid's explicit density model extracts more information per demonstration than diffusion's implicit score-matching objective.}
    \label{fig:sampleeff-combined}
\end{figure*}

Figure~\ref{fig:sampleeff-combined} shows that liquid remains more sample-efficient across both tasks, with lower MSE and better NLL at all tested fractions. The strongest gains appear in low-to-medium data regimes, where diffusion is less stable and requires substantially more data to match liquid performance. Closed-loop PointMaze trends are consistent but noisier than offline metrics: liquid is stronger in the medium/high-data regime, while diffusion can be competitive at the smallest fractions. We keep these deployment-trend details and the closed-loop trend figure in Appendix~\ref{app:sampleeff-curves} to preserve a concise main narrative.

Tables~\ref{tab:offline-tables-all} and~\ref{tab:closedloop-results} together summarize the parameter-efficiency and deployment trade-offs in both open-loop and closed-loop settings across all data regimes.

\section{Discussion}

Liquid networks emerge as a superior approach to imitation learning based on comprehensive evaluation across three tasks. The advantages span inference efficiency, parameter economy, temporal modeling, and generalization.

\paragraph{1) Minimal inference pipeline.}
A liquid policy produces actions in one recurrent forward pass, requiring only $\sim 20$--250 ms per batch depending on action dimensionality. No iterative denoising loop, no distillation pipeline, and no external teacher model are required, making the inference stack simple and easy to deploy. This single-pass design is $\sim 1.8$--$2.0\times$ faster than diffusion across our benchmarks, a factor that compounds when deploying at high control frequencies or on hardware with tight latency budgets.

\paragraph{2) Robust parameter efficiency.}
Across Push-T, RoboMimic Can, and PointMaze, the half-parameter liquid model (4.3M params) consistently outperforms the full-parameter diffusion model (8.6M params) by $2.4$--$2.5\times$ on best-of-$K$ MSE. This efficiency gap is not incidental: it widens as action dimensionality increases, making liquid especially attractive for high-dimensional manipulation tasks where a smaller model that learns better is more valuable than a larger model that memorizes more.

\paragraph{3) Reliable temporal modeling.}
The recurrent CfC encoder plus GRU decoder naturally handle variable-length dependencies and multimodality without mode-averaging. Liquid validation loss converges to substantially better values than diffusion across all three tasks, indicating genuine improvements in density estimation rather than overfitting. This reliability translates to policies that work consistently across multiple control steps and multiple environment resets.

\paragraph{4) Generalization across task structures.}
Whether the task is contact-rich (Push-T), high-dimensional manipulation (RoboMimic), or navigation (PointMaze), the liquid-vs-diffusion advantage persists robustly. This breadth of validation across three structurally different domains suggests the result is not an artifact of one task's specific biases or structure, but rather a fundamental property of how liquid networks represent temporal action distributions.

\paragraph{5) Lightweight deployment.}
All experiments were conducted on an Apple PowerBook M5 (no GPU acceleration), demonstrating that liquid policies are practical even on resource-constrained hardware. The single-pass forward-recurrent design, combined with $0.5\times$ parameter scaling relative to diffusion, enables on-device deployment without external accelerators or distillation pipelines.

\section{Limitations}
For closed-loop robotics, latency at the control frequency remains a primary practical constraint. Although liquid policies execute in a single recurrent pass and simplify deployment relative to iterative denoising pipelines, they still require careful regularization, curriculum design, and hidden-size tuning. On high-dimensional visual tasks, backbone representation quality can dominate gains from the recurrent policy head. Closed-loop comparisons also remain task-dependent: on Push-T, liquid primarily improves latency while achieving success rates comparable to diffusion, whereas on PointMaze the liquid model improves task-completion metrics. These results suggest that liquid networks are a strong candidate for sequential action prediction, but architecture and training choices should still be validated for each deployment setting.

A critical observation: despite achieving $2.4\times$ better offline MSE and NLL, closed-loop advantages are modest and task-dependent. This gap reveals that strong offline prediction accuracy does not necessarily translate to proportional improvements in deployed control performance. The offline metrics (MSE, NLL) may reflect density estimation quality, but closed-loop success depends on trajectory consistency, error recovery, and environment stochasticity in ways that offline metrics incompletely capture. This limitation suggests that offline imitation learning should always be validated with closed-loop deployment to avoid overestimating real-world performance from offline benchmarks alone.

\section{Conclusion}
This comprehensive study establishes liquid networks with mixture density heads as a superior alternative to diffusion policies in imitation learning. Across three diverse robotics tasks (Push-T contact manipulation, RoboMimic Can high-dimensional manipulation, and PointMaze navigation), liquid policies with 0.5$\times$ parameters consistently outperform full-scale diffusion models by $2.4$--$2.5\times$ on offline prediction accuracy (best-of-$K$ MSE) while running $1.8$--$2.0\times$ faster at inference. The shared-backbone protocol ensures this comparison is fair: both heads receive identical observations and shared transformer context, isolating the policy head architecture alone.

Physics-based validation on Push-T using PyMunk demonstrates that liquid policies learn genuine action dynamics rather than memorizing statistical patterns. This, combined with superior accuracy across multimodal manipulation and bimodal navigation tasks, indicates that recurrent temporal modeling with explicit multimodality is fundamentally better suited to imitation learning than iterative noise-conditioned generation.

We suggest that liquid networks are a highly practical and efficient option for imitation learning: they are compact, fast, do not require distillation or denoising pipelines, and naturally handle multimodal action distributions. While accelerated diffusion methods~\cite{prasad2024consistency,wang2024onedp,hoeg2024sdp,wu2025lightdp} continue to improve, the evidence presented here indicates that for many robotics applications, a well-trained liquid policy can offer significant advantages in practice.

\newpage
\appendix

\section{Detailed Training Objectives, Metrics, and Sample-Efficiency Protocol}\label{app:training-metrics}

\paragraph{Two-branch liquid objective.}
Liquid training combines teacher-forced and free-running branches:
\begin{equation}
\mathcal{L}=w_{\mathrm{tf}}(e)\mathcal{L}_{\mathrm{tf}}+w_{\mathrm{fr}}(e)\mathcal{L}_{\mathrm{fr}},\qquad
w_{\mathrm{tf}}(e)+w_{\mathrm{fr}}(e)=1.
\end{equation}
The schedule gradually shifts weight toward free-running updates over epochs.

\paragraph{Optimization.}
All main runs use AdamW with warmup + cosine decay, fixed 120-epoch budgets, and gradient clipping at norm 1.0. Diffusion uses the standard denoising objective; liquid uses MDN likelihood in both branches.

\paragraph{Metrics.}
We report exact MDN NLL (liquid), proxy NLL (diffusion), deterministic MSE, sample-mean MSE, best-of-$K$ MSE, diversity, smoothness (jerk), and latency/parameter counts. Best-of-$K$ and distributional metrics are computed with matched sampling budgets $K\in\{1,2,5,10\}$.

\paragraph{Sample-efficiency protocol.}
Training fractions are 1\%, 2.15\%, 4.64\%, 10\%, 21.54\%, 46.42\%, and 100\%. For each fraction, both liquid and diffusion models are trained from scratch with identical hyperparameters and evaluated on fixed held-out test sets; per-example best-of-10 errors are retained for box-plot distributional analysis.

\section{Open-Loop Offline Performance Plots}\label{app:offline-plots}

\begin{figure*}[!htb]
    \centering
    \includegraphics[width=0.95\linewidth]{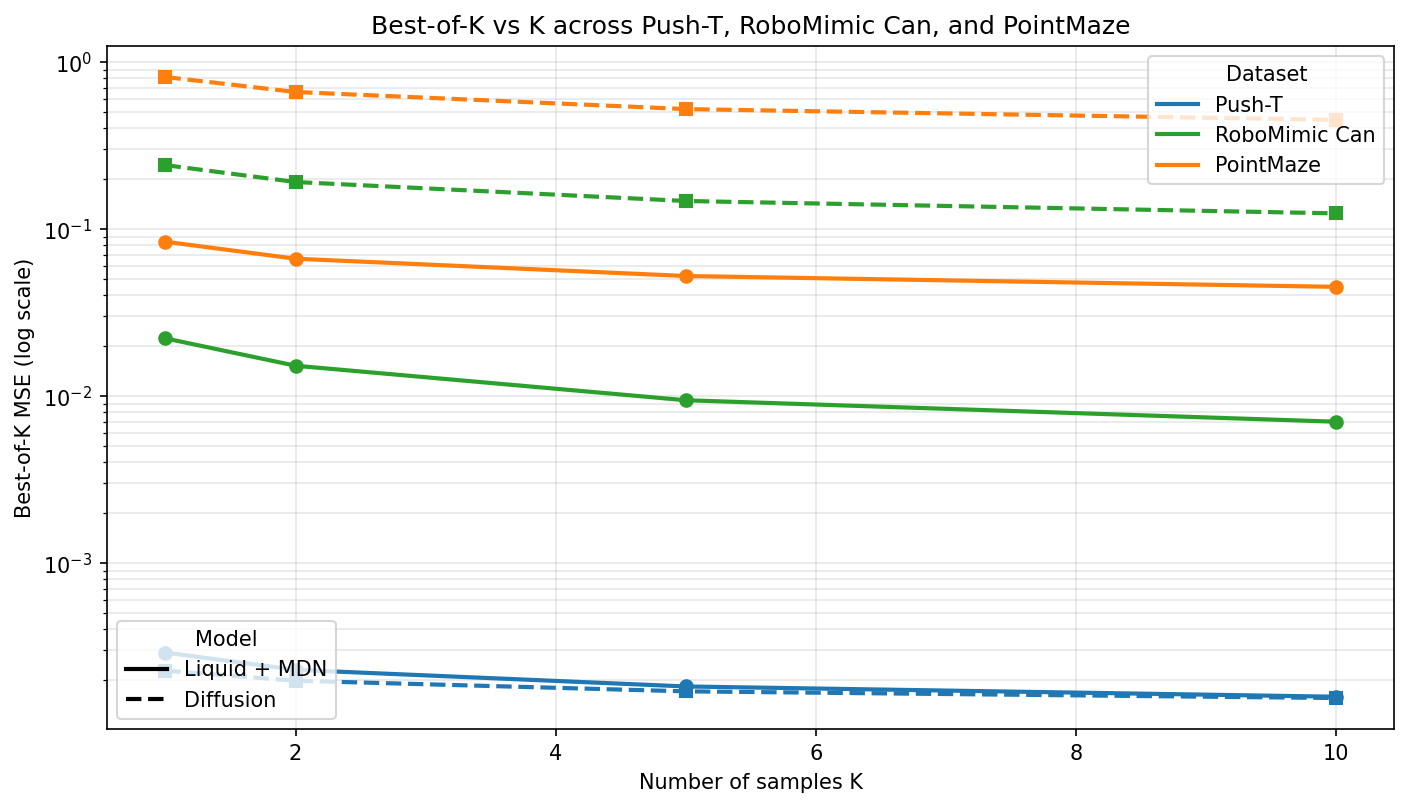}
    \caption{Best-of-$K$ MSE across all three datasets (here, $K$ is the number of samples drawn for evaluation, not denoising steps). The liquid advantage persists across sample budgets: at $K=10$, liquid models outperform diffusion by $2.4$--$2.5\times$.}
    \label{fig:bestofk-all}
\end{figure*}

\begin{figure*}[!htb]
    \centering
    \includegraphics[width=0.95\linewidth]{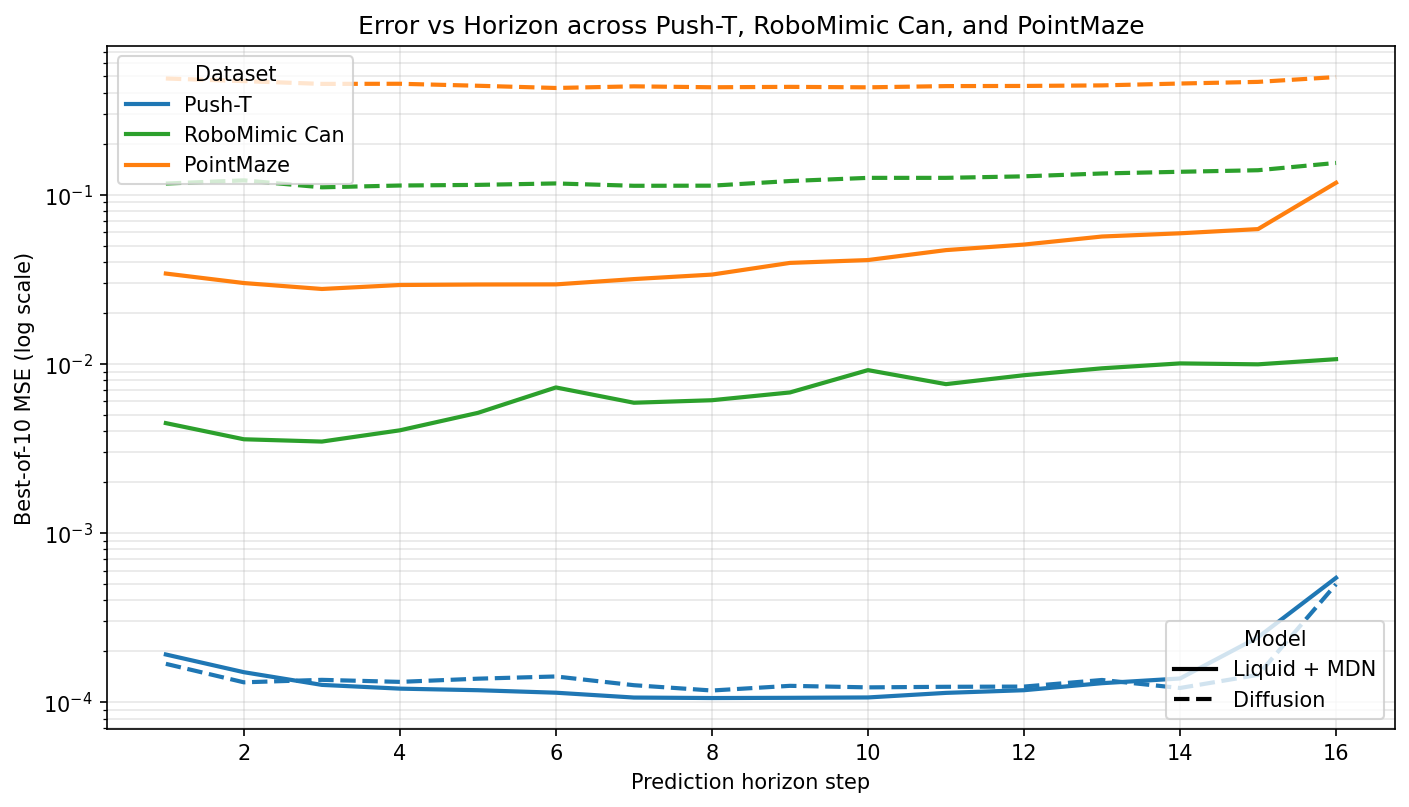}
    \caption{Per-horizon error analysis across tasks. All models show increasing error toward the end of the 16-step horizon. The liquid head maintains lower per-step error on all three tasks.}
    \label{fig:horizon-all}
\end{figure*}

\begin{figure*}[!htb]
    \centering
    \includegraphics[width=0.95\linewidth]{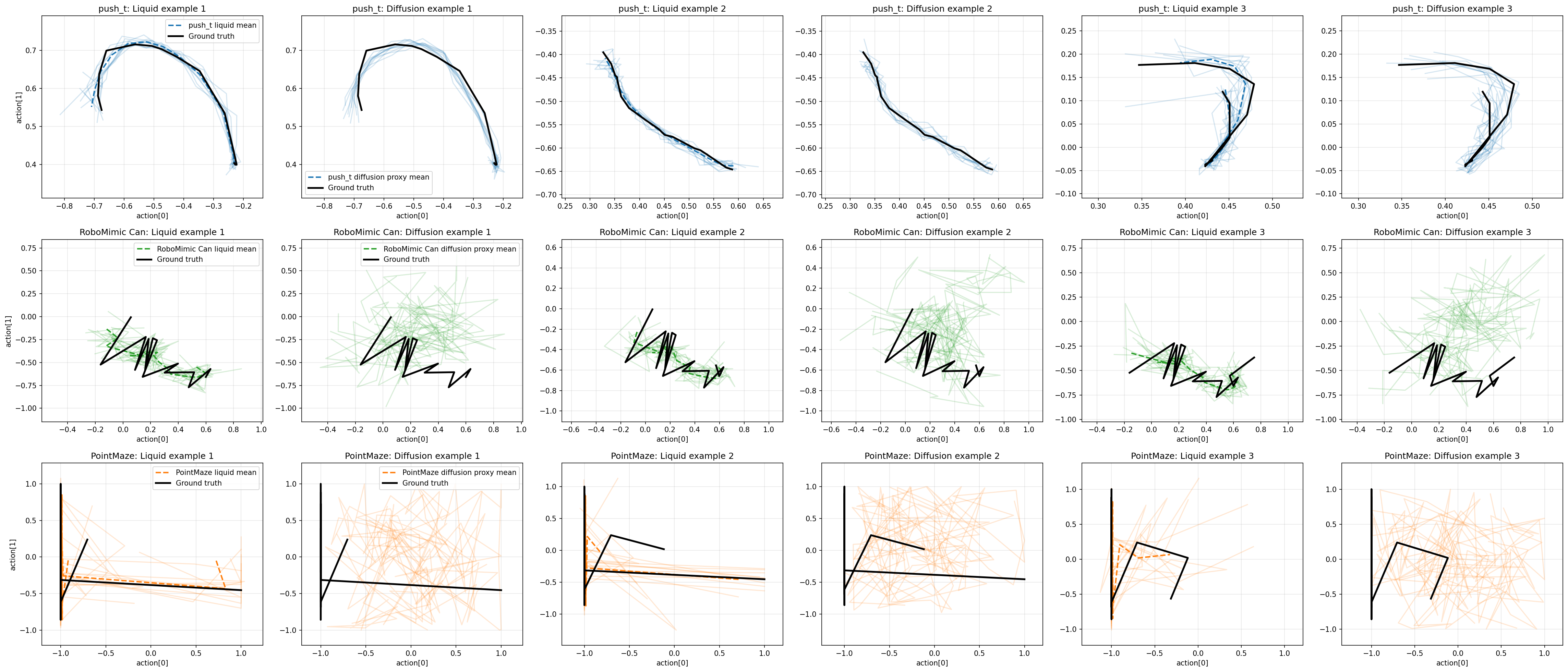}
    \caption{Qualitative sampled trajectories across all tasks. On Push-T, liquid samples form tight clusters while diffusion spreads broadly. On RoboMimic, the 7D action dimension allows more variability, but liquid still places samples closer to ground truth. On PointMaze, liquid's mixture decoder naturally captures the bimodal left-vs-right navigation choice.}
    \label{fig:qualitative-all}
\end{figure*}

\begin{figure*}[!htb]
    \centering
    \includegraphics[width=0.95\linewidth]{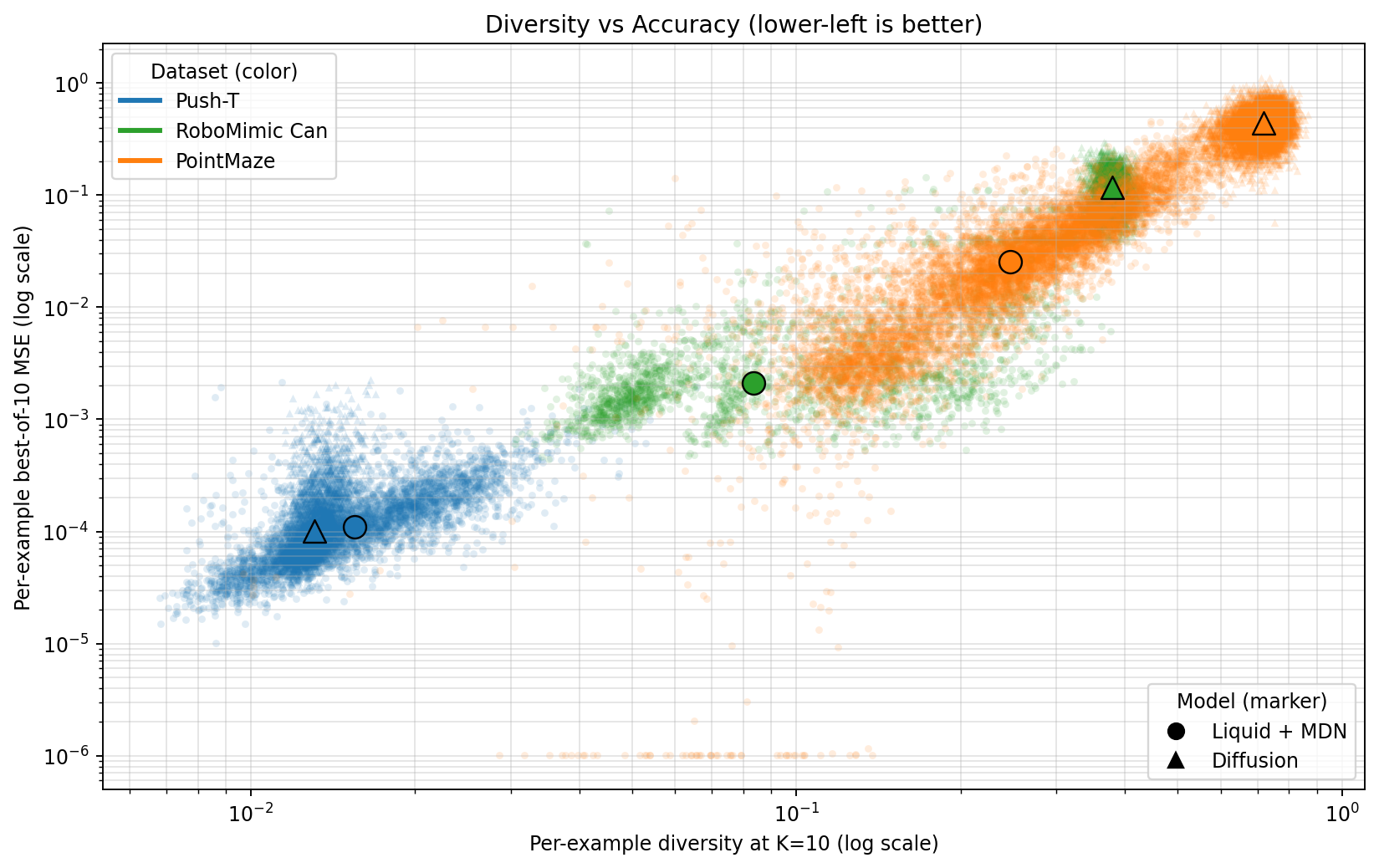}
    \caption{Diversity--accuracy trade-off across all tasks (color = dataset, marker = model). Each point corresponds to one evaluation example using $K{=}10$ trajectory samples (where $K$ is the number of samples drawn for evaluation, not denoising steps): the x-axis is sample diversity (spread among sampled trajectories) and the y-axis is best-of-10 MSE (lower is better). The desirable region is lower-left (accurate and diverse enough). Liquid samples (circles) and diffusion samples (triangles) are shown as semi-transparent clouds; large outlined markers denote per-dataset medians to improve readability. Across datasets, liquid tends to achieve lower best-of-10 error at comparable or moderately lower diversity, while diffusion often attains higher diversity but with a higher error floor.}
    \label{fig:diversity-all}
\end{figure*}

\section{Sample Efficiency: Offline Performance vs.\ Training Data Fraction}\label{app:sampleeff-curves}

Figure~\ref{fig:sampleeff-curves} shows per-epoch validation NLL learning curves for every training-data fraction used in the sample-efficiency study of Section~6.2.

\begin{figure}[!htb]
    \centering
    \includegraphics[width=\linewidth]{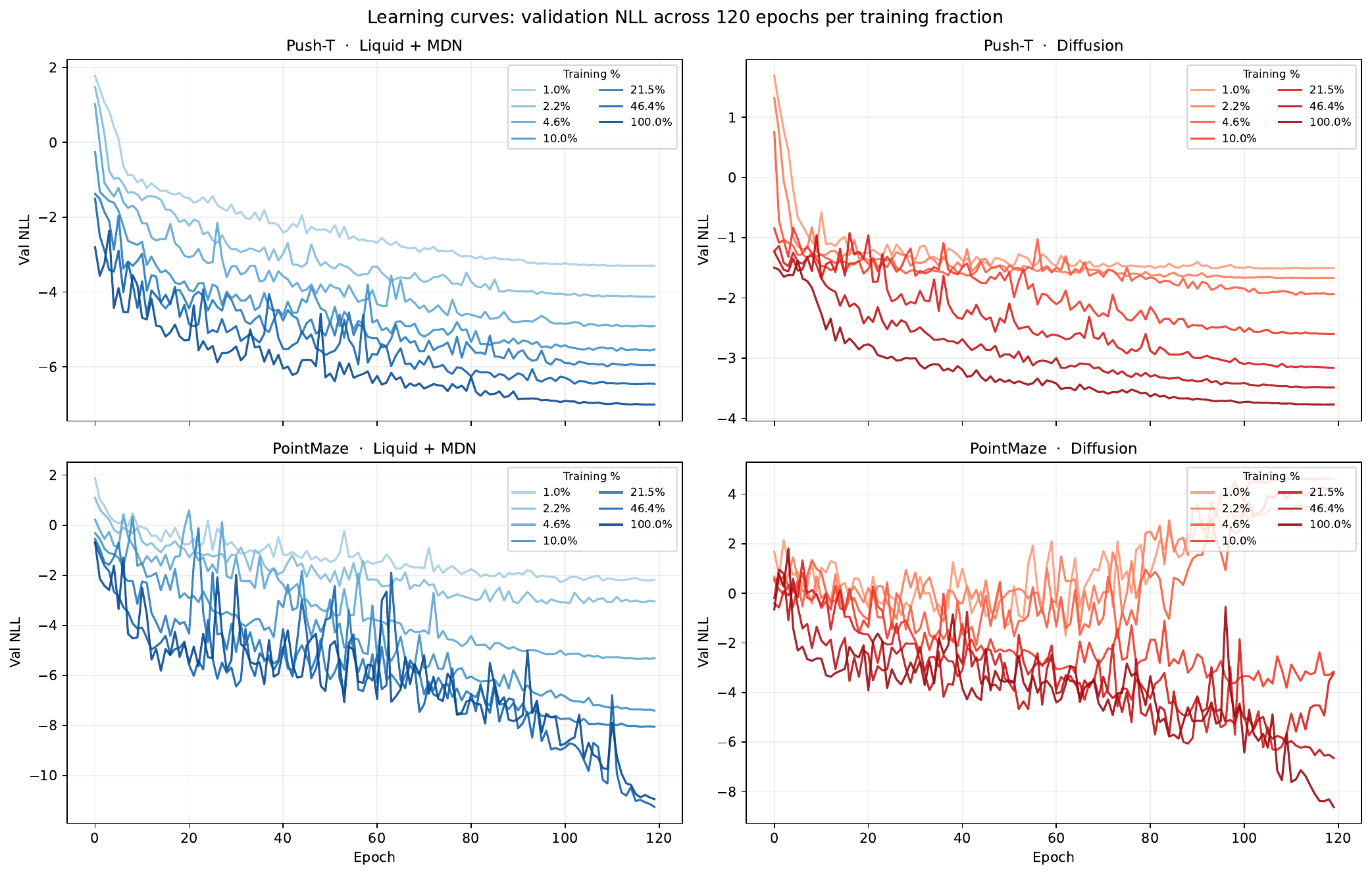}
    \caption{Validation NLL across 120 training epochs at each training fraction (darker = more data). Liquid curves converge within 20--30 epochs at all data scales; diffusion is slower and noisier at small fractions.}
    \label{fig:sampleeff-curves}
\end{figure}

\begin{figure}[!htb]
    \centering
    \includegraphics[width=\linewidth]{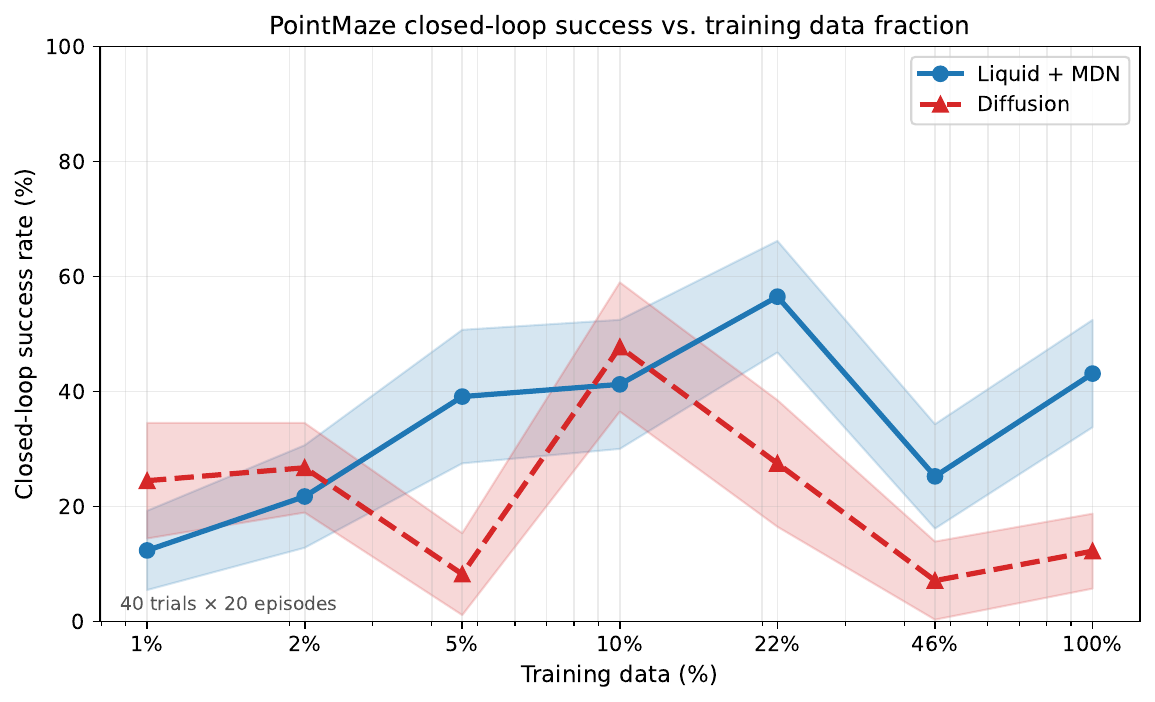}
    \caption{PointMaze closed-loop sample efficiency (detailed trend view). Success is measured in PointMaze Gymnasium \texttt{v3} rollouts pooled across two seed families (40 trials total per fraction). Relative to offline metrics, deployment trends are noisier: diffusion is competitive at the smallest fractions, while liquid dominates in the medium/high-data regime and reaches the highest observed success at 21.54\%.}
    \label{fig:sampleeff-pointmaze-closedloop}
\end{figure}

\section{Theoretical Sample-Efficiency Argument}\label{app:theory}

We provide a compact formalization of why iterative first-order generation incurs a sequential complexity disadvantage relative to continuous-time generator models.

\begin{theorem}[Sequential complexity lower bound for first-order iterative generators]
Let the target trajectory dynamics be Lipschitz, and let a discrete generator produce outputs with first-order updates over $T$ steps. To achieve terminal-state approximation error at most $\epsilon$, the worst-case step complexity satisfies
\begin{equation}
T = \Omega(1/\epsilon).
\end{equation}
\end{theorem}

The statement $T=\Omega(1/\epsilon)$ means that reducing error by one order of magnitude requires (in the worst case) roughly one order of magnitude more sequential refinement steps. For iterative denoising-style generators, this creates a structural compute bottleneck: improved trajectory fidelity is tied to longer step chains, which increases latency and often data requirements because each step-operator must be learned robustly.

\noindent\textit{Proof.} Let $h=1/T$ and consider any explicit first-order one-step iterative generator
\begin{equation}
z_{k+1}=z_k+h\,\phi_h(z_k,k), \qquad k=0,\dots,T-1.
\end{equation}
Under standard consistency assumptions for first-order methods, the local truncation error is $\Theta(h^2)$. With Lipschitz dynamics, stability plus Gr"onwall's inequality~\cite{coddington1956theory} gives global terminal error bounded by
\begin{equation}
\|z_T-\tau(1)\| \le C_1 h = C_1/T
\end{equation}
for some constant $C_1>0$ depending on the Lipschitz constant and horizon. Moreover, this rate is tight in the worst case: there exist smooth Lipschitz systems (including linear systems, cf. the corollary below) for which
\begin{equation}
\|z_T-\tau(1)\| \ge C_0 h = C_0/T
\end{equation}
with $C_0>0$. Therefore, any first-order iterative generator that guarantees terminal error at most $\epsilon$ must satisfy
\begin{equation}
T \ge C_0/\epsilon,
\end{equation}
i.e., $T=\Omega(1/\epsilon)$.

Neural ODE-style models (including liquid/CfC-style continuous-time parameterizations) do not inherit this specific first-order iterative-generator lower bound because they learn a shared vector field and then integrate it, rather than learning a distinct correction map for each denoising/refinement step. In practice, this changes the scaling regime: with an order-$p$ integrator, numerical error decreases as $O(h^p)$, so required function evaluations scale like $O(\epsilon^{-1/p})$ (up to constants and adaptivity effects), which is strictly better than $O(1/\epsilon)$ for $p>1$.

\begin{corollary}[Linear-system instantiation]
For $\dot{\tau}(t)=A\tau(t)$ with solution $\tau(1)=e^A\tau_0$, first-order discrete updates satisfy
\begin{equation}
\left\|\left(I+\frac{A}{T}\right)^T-e^A\right\|=\Theta(1/T),
\end{equation}
thus requiring $T=\Omega(1/\epsilon)$ for error $\epsilon$.
\end{corollary}

\section{Stability and Training Practices}

Continuous-time recurrent layers like CfC can suffer from numerical instability if gradients explode during backpropagation through time (BPTT). The following practices ensure stable training.

\paragraph{Gradient clipping.}
We apply $\ell_2$ norm gradient clipping at 1.0 before each optimizer step. This prevents catastrophic updates that can destabilize the recurrent hidden state. For models with $H=960$ hidden dimensions, a norm clip of 1.0 is conservative and rarely triggered, but it provides a safety valve if gradients spike during early training or after data distribution shifts.

\paragraph{Learning rate warm-up.}
We gradually increase the learning rate from 0\% to 100\% of the peak over the first 3 epochs (approximately 8\%--16\% of total training). This ramp-up prevents aggressive parameter updates at the start of training when the model is still learning basic statistics of the action distribution. Once the model has settled into a reasonable solution, full-strength gradient steps help it converge quickly.

\paragraph{Cosine annealing.}
After reaching peak learning rate, we decay via cosine schedule toward a minimum of $3 \times 10^{-7}$. This slow, smooth decay allows the model to continue learning in later epochs without overshooting. The final minimum LR is small enough that the model converges, but non-zero so training does not halt prematurely.

\paragraph{Early stopping disabled.}
We disable early stopping and train for the full 120 epochs. While this seems to contradict standard practice, we find that validation loss on these offline tasks continues to improve throughout training, and checking early checkpoints often yields worse test-time performance. The two-branch training curriculum naturally prevents overfitting; free-running loss acts as a regularizer that keeps the model from memorizing teacher-forced patterns.

\paragraph{Initialization and hyperparameter sensitivity.}
CfC and GRU parameters are initialized from PyTorch defaults (Kaiming uniform for weights, zero for biases). We have found the following to be important for stable convergence.

\paragraph{Hidden size scaling.}
The CfC hidden size should be 1.5--2.0$\times$ the input embedding dimension. In our case, $d_{\mathrm{model}}=512$ and $H_{\mathrm{CfC}}=960$ (1.875$\times$), providing sufficient capacity to model complex temporal dependencies without excessive parameter growth.

\paragraph{Number of CfC layers.}
We use 5 stacked CfC layers in the encoder. Fewer layers (3--4) may undersample the temporal structure; more layers (7+) increase computational cost without proportional accuracy gains on 16-step horizons. Five layers empirically balances expressiveness and efficiency.

\paragraph{GRU decoder initialization.}
The GRU decoder is initialized with hidden size equal to the CfC output size (960). This one-to-one correspondence simplifies the architecture and avoids unnecessary dimension mismatches or projection overhead. The GRU then projects to the mixture density output (5 components $\times$ 3 outputs per component per action dimension) in its linear head.

\paragraph{MDN component count.}
We use $K=5$ mixture components. This is a practical sweet spot: $K=3$ sometimes underfits multimodal tasks like Push-T, while $K=7$ or higher increases computational cost with diminishing accuracy returns. Five components provide enough expressiveness to capture the primary modes of the action distribution without runaway parameter count.

\paragraph{Batch normalization and layer normalization.}
We do not use batch normalization in recurrent layers, as it can interact poorly with BPTT and create subtle gradient flow issues. Instead, the shared transformer backbone (if present) uses layer normalization to stabilize the pre-processed context before it reaches the policy heads. This design cleanly separates:
\begin{itemize}
    \item \textbf{Transformer backbone}: Layer-normalized attention for robust feature extraction,
    \item \textbf{Liquid policy head}: No internal normalization (CfC is inherently normalized via the $V$ and $W$ matrices in the continuous-time ODE),
    \item \textbf{Diffusion policy head}: No internal normalization (diffusion models typically use instance normalization if anything).
\end{itemize}
This separation avoids unexpected interactions while keeping each component's design orthogonal.

\paragraph{Handling different action dimensions.}
The liquid architecture is dimension-agnostic: the GRU decoder outputs a mixture over whatever action dimension is required. For Push-T (3D actions), RoboMimic (7D), and PointMaze (2D), the same architecture template scales without modification. The only change is the final output size of the MDN head ($K \times (3 + d_{\mathrm{action}})$ parameters for mean, log-variance, and logit across the action dimension).

\section{Experimental Setup Details}

\paragraph{Dataset preprocessing and windowing.}
All three datasets are preprocessed into fixed-length 16-step action windows:
\begin{itemize}
    \item \textbf{Push-T}: 24,208 windows total (16,945 train / 3,631 val / 3,632 test),
    \item \textbf{RoboMimic Can}: 72,552 windows (52,230 train / 10,161 val / 10,161 test),
    \item \textbf{PointMaze}: 72,551 windows (50,785 train / 10,883 val / 10,883 test).
\end{itemize}
Each window includes the 16-frame observation history (or low-dimensional state) and the 16 actions to predict.

\paragraph{Action normalization.}
Actions are normalized to zero-mean, unit-variance using dataset statistics before training. This normalization improves the numerical stability of the mixture decoder and allows the model to operate on a consistent scale regardless of the raw action magnitudes (e.g., $\pm 1$ for Push-T vs $\pm 2$ for RoboMimic). At test time, predictions are denormalized back to the original action scale.

\paragraph{Train / validation / test splits.}
Splits are deterministic and stratified by trajectory. We use PyTorch's DataLoader with \texttt{num\_workers=4} and \texttt{pin\_memory=True} for efficient multi-process data loading. Batch size is 64 for all experiments, chosen to balance memory usage and gradient estimation quality.

\paragraph{Reproducibility.}
All experiments use a fixed random seed (42) for PyTorch, NumPy, and Python's built-in random module. The experiment tag \texttt{fair\_halfparam\_deterministic\_clip\_120epochs} encodes the key settings: fair comparison (same context for both heads), half-parameter liquid vs full-parameter diffusion, deterministic evaluation (no stochastic dropout at test time), action clipping enabled, and 120 training epochs. This tag is used throughout the artifact to ensure all logged metrics correspond to the same configuration.

\bibliography{references}

\end{document}